\documentclass[final,12pt]{colt2025} 

\title[Learning Algorithms in the Limit]{Learning Algorithms in the Limit}
\usepackage{macros_hristo}
\usepackage{times}

\coltauthor{\Name{Hristo Papazov} \Email{hristo.papazov@epfl.ch} \\ \Name{Nicolas Flammarion} \Email{nicolas.flammarion@epfl.ch} \\
\addr EPFL, Lausanne, Switzerland}

\begin{document}

\maketitle

\begin{abstract}
This paper studies the problem of learning computable functions in the limit by extending Gold’s inductive inference framework to incorporate \textit{computational observations} and \textit{restricted input sources}. Complimentary to the traditional Input-Output Observations, we introduce Time-Bound Observations, and Policy-Trajectory Observations to study the learnability of general recursive functions under more realistic constraints. While input-output observations do not suffice for learning the class of general recursive functions in the limit, we overcome this learning barrier by imposing computational complexity constraints or supplementing with approximate time-bound observations. Further, we build a formal framework around observations of \textit{computational agents} and show that learning computable functions from policy trajectories reduces to learning rational functions from input and output, thereby revealing interesting connections to finite-state transducer inference. On the negative side, we show that computable or polynomial-mass characteristic sets cannot exist for the class of linear-time computable functions even for policy-trajectory observations.
\end{abstract}



\section{Introduction}\label{sec:introduction}

Nearly a century ago, \citet{hilbert1928grundzuge} formulated ``the most general problem of mathematics'' -- the so-called \textit{Entscheidungsproblem}\footnote{See \citep{sep-entscheidungsproblem} for a brief history of the problem.} -- which sought an \textit{effective procedure} for deciding the provability of any first-order logical sentence from a given set of axioms. Despite Hilbert's steadfast optimism, progress eluded the Entscheidungsproblem for almost a decade because, perhaps ironically, ``the main problem of mathematical logic'' lacked a mathematical formulation. Indeed, the term \textit{effective procedure} stubbornly evaded formal treatment and only represented the intuitive idea of a finitely described finite process whose execution requires no spark of creativity. The breakthrough came with the work of \citet{church1936note} and \citet{turing1936computable}, who formalized the intuitive predicate \textit{effective} through \textit{$\l$-calculability} and \textit{Turing-computability}, respectively. This paradigm shift ultimately led to the resolution of the Entscheidungsproblem in the negative.

Interestingly, the death of Hilbert's program elucidated a remarkable coincidence in mathematics: Every effort to formalize the intuitive notion of an \textit{effectively computable} function appears to single out exactly the same set of functions. Initially proved for $\l$-calculability and Turing-computability, this extensional equivalence now captures all known reasonable\footnote{Some unreasonable formal models of computation, which lead to hypercomputation, utilize infinite precision arithmetic of real numbers and mechanisms for accelerating computation, see \citep{sep-computation-physicalsystems}. We refer the curious reader to \citep{aaronson2005guest} for a survey of how hypercomputation clashes with the known laws of physics.} formalisms of computation. 
Due to the sheer and exhaustive abundance of equivalent models of computation, nowadays, computer scientists, philosophers, and mathematicians widely accept the \textbf{Church-Turing Thesis} (CTT) which states that Turing machines (TMs) can \textit{simulate} all ``intuitively computable'' functions/processes. In our paper, we adopt the prevailing view that the umbrella term ``intuitively computable'' encompasses other informal notions such as ``effectively computable,'' ``computable by an algorithm,'' and ``computable by a harnessable physical process.''\footnote{For a broad disambiguation between different interpretations of the CTT, we refer the reader to \citep{sep-computation-physicalsystems, sep-church-turing}. Remarkably, the CTT lives at the intersection of mathematics, physics, and philosophy.} Hence, we take for granted that a TM can simulate any reasonable computational model (see \Cref{app:sec:ctt} for a discussion on simulations and the CTT).

\paragraph{Learning Computable Functions in the Limit.} 
Now, a computational model $M$ working on an enumerable domain $\D$ may not halt for all inputs $x \in \D$. Thus, when considering computation at the function level, we make a distinction between the family of total recursive functions $\CC_\D$ -- computable on the full domain $\D$, and the family of general recursive functions $G_\D$ -- computable on subsets of $\D$.\footnote{Note that $\CC_\D \subsetneq G_\D$ as proved by \cite{turing1936computable} through the undecidability of the Halting Problem.}
Having briefly argued that every function computable by the laws of nature or the complex processes of the human brain ultimately resides in $G_\D$ for the relevant domain $\D$, in this paper, we set out to prove that the family of general recursive functions are learnable in the limit from \textit{computational observations}.

In simplest terms, the learning-in-the-limit framework of \citet{gold1967language} refers to the scenario where a learning algorithm $\L$ at time $t \in \N$ observes a piece of information $i_t$ about an unknown function $f$ from a concept class $\O$ and comes up with a hypothesis based on everything observed so far: $h_t \leftarrow \L(i_1, \dots, i_t)$. If for every ground truth $f \in \O$ and every valid sequence of observations $(i_t)_{t=1}^\infty$, the learner always converges on the correct function after finitely many mistakes (i.e., $h_t = f, \ \forall t \geq t^\star$), 
then we say that $\L$ learns $\O$ in the limit and define $\O$ as learnable in the limit.

Though concise, the above description left out some important details. First, clearly $\O \subseteq G_\D$ for some enumerable domain of computation $\D$. Second, with the notation $h_t \leftarrow \L(i_1, \dots, i_t)$, we meant that the learning algorithm at time $t$ outputs a representation $R_t$ belonging to a computable set of representations $\RR$\footnote{For example, one could think of $\RR$ as the set all finite-state automata, cellular automata, or TMs.} such that $R_t$ generates the function $h_t$ when executed on a simulator, like a Universal TM. Third, each example $i_t = (x_t, f(x_t), \a(M, x_t))$ includes an input $x_t$, a function value $f(x_t)$, and auxiliary information $\a(M, x_t)$\footnote{Without $\a(M, x)$, we arrive at \cite{gold1967language}'s setup of \textit{learning in the limit with an informant}.} about the computation $x_t \mapsto f(x_t)$ by some computational model $M$. Crucially, any valid sequence of observations $(i_t)_{t=1}^\infty$ must exhaust all possible inputs: i.e., $\forall x \in \D \ \exists t \in \N$ s.t. $i_t = (x, f(x), \a(M, x))$.\footnote{We relax this requirement in our paper. Instead, if all observed inputs come from an input source $I \subseteq \D$, we will require a successful learner $\L$ to learn a representation for a function $h$ in the limit s.t. $h_{|I} = f_{|I}$.} Finally, the requirement that $\L$ must correctly learn $f$ in the limit for any ordering of the example set  $\{(x, f(x), \a(M, x)) : x \in \D\}$ exists to facilitate genuine learning. Otherwise, an ordering function could encode $M$ as the first input, allowing $\L$ to learn $f$ without meaningful generalization.

\paragraph{Characteristic Sets.}
Learning-in-the-limit algorithms should correctly identify concepts for any adversarial ordering of examples. Nevertheless, one could still study the sample efficiency of such algorithms in the presence of favorable curricula through the notion of \textit{characteristic sets}. First introduced by \citet{gold1978complexity} for finite-state transducers (FSTs) and later generalized by \citet{de1997characteristic} for arbitrary concept classes, characteristic sets represent a core set of inputs linked to a representation which once observed allow the learner to identify the ground truth.

\paragraph{Additional Information.} 
Gold proved that, when the observations include no additional information about the computational process taking $x$ to $f(x)$ (i.e., $\a(M, x) = \eps$ -- the empty string), no algorithm can learn the concept class of totally recursive functions $\CC_\D$ in the limit \citep[Theorem~I.5]{gold1967language}. Although, we revisit this input-output-observations (IOOs) framework in \Cref{sec:blackbox}, we believe that intelligent agents hardly ever learn new skills from such impoverished observations. In the real world, learning systems observe computational processes, not the mere results of computation. For example, in human learning, teachers guide students through the process of reaching the correct answer and explain numerous intermediate steps. Similarly, research in machine learning (\citet{wei2022chain, anil2022exploring}) demonstrates that both pretraining and in-context prompting with chain-of-thought (CoT) data significantly improve reasoning accuracy and facilitate length generalization. Motivated by these considerations, we propose two types of \textit{computational observations} as an extension to the original framework:
\begin{itemize}[label = $\circ$, topsep=0pt, left=5pt, itemsep=0pt, labelsep=5pt]
    \item \textit{Time-Bound Observations (TBO).} 
    The learner $\L$ only gains a general sense of the task's hardness through a rough upper bound $\atb(M, x) \in \N$ on the number $t_M(x)$ of discrete computation steps the computational model $M$ takes to compute $f(x)$. That is, we assume that up to some model-dependent scale, the upper bound $t_M(x) = O_M(\atb(M, x))$ holds.
    \item \textit{Policy-Trajectory Observations (PTO).} The learner $\L$ observes the full interaction of the computational agent $M$ with the environment without observational access to $M$'s internal state updates. Instead, $\L$ perceives only external behavior -- analogous to a student looking at a teacher write on a blackboard without access to the teacher's neuronal activity. Thus, $\a_{\mathrm{PT}}(M, x)$ represents the external, readily observable part of $M$'s computational trajectory.\footnote{In \Cref{sec:CoT}, we will draw connections to Learning from Demonstrations -- a subgenre of Reinforcement Learning, and show that $\apt(M, x)$ truly constitutes observations of trajectories of the agent's policy.}
    
\end{itemize}
Note that PTOs provide strictly more information than TBOs since the learner observes the exact number of computational steps. For our \textit{TBO-Learning} results, we allow observations over arbitrary computational models. However, extracting meaningful information from PTOs requires a fixed model family. Indeed, inferences from external behavior to internal computation only work if the learner knows the model architecture in advance. Thus, all our \textit{PTO-Learning} results concern TMs.

\paragraph{Paper Contributions and Organization.}
\Cref{sec:preliminaries} develops the formal foundation, defining computational models and agents; learning in the limit with $\a$-observations and restricted inputs; and learning efficiency through characteristic sets. The next sections contain our main contributions:
\begin{itemize}[topsep=0pt, left=5pt, itemsep=0pt, labelsep=5pt]
    \item \textit{Learning from IOOs.} In \Cref{sec:blackbox}, we generalize the learnability results from \citep{gold1967language} by proving that any parameterized complexity class of general recursive functions is learnable in the limit from IOOs. In particular, polynomial-time computable functions are learnable in the limit.
    \item \textit{Learning from TBOs.} In \Cref{sec:clock}, we prove using the Extended Church-Turing Thesis (ECTT) that regardless of the computation domain $\D$ and the model $M_f$ computing the ground-truth $f$, the concept class $G_\D$ is learnable in the limit from TBOs. Notably, without referencing the ECTT, we prove that the family of Turing-computable functions is learnable in the limit from TBOs.
    \item \textit{Learning from PTOs.} In \Cref{sec:CoT}, we prove that no algorithm for learning the family of Turing-computable functions in the limit can have polynomial characteristic sets even with PTOs. Additionally, we reduce the problem of learning general recursive functions  with PTOs to the problem of learning FSTs from input-restricted IOOs.
\end{itemize}
Finally, \Cref{sec:conclusion} concludes with implications and open problems, and the appendices include detailed technical proofs and supporting discussion. 

\paragraph{Related Work.}
Our work builds upon prior research in formal language learning and recent progress in learning computable functions. Here, we outline key contributions relevant to our framework.

Concerning FSTs, \citet{gold1978complexity} proved that the class of finite-state transducers is \textit{identifiable in the limit from polynomial time and data} (IPTD). Subsequently, \citet{oncina1992inferring} introduced the RPNI algorithm -- a polynomial-time state-merging method for learning FSTs in the limit. Thereafter, \citet{de1997characteristic} formalized the notion of characteristic sets, demonstrating that context-free grammars and other formal classes are not IPTD. Furthermore, \cite{de1997characteristic} noticed that the RPNI algorithm implies that FSTs are polynomially $T/L$-teachable -- a stronger learnability condition developed by \cite{goldman1993teaching}. Building on this line of work, \citet{parekh2001learning} proved that FSTs are PAC-learnable under simple distributions and established several equivalences between active and passive learning frameworks.

Recently, \citet{malach2024auto-regressive} showed that any computable function over fixed-length binary inputs is PAC-learnable by a linear autoregressive model when trained on CoT sequences. We extend this result by proving that computable functions over arbitrary input lengths are exactly learnable from a finite number of CoT observations. 
Finally, \citet{kleinberg2024language} revisited the learning-in-the-limit framework in the context of large language models, highlighting that \textit{generation is easier than identification} through a proof that learning to generate from a ground truth language is possible—under broad conditions—when only observing positive examples. 


\section{Formal Framework} \label{sec:preliminaries}
In this section, we build the formal bedrock for our learning-in-the-limit results. We group the topics thematically into two subsections dedicated to computational models and learning theory, respectively. Before proceeding, we introduce key notation. Let $\eps$ denote the empty string. For any set $S$, the Kleene plus $S^+$ denotes $\cup_{n=1}S^n$, while the Kleene star $S^\star$ stands for $\{\eps\} \cup S^+$. For an enumerable domain $\D$, we define the families of general recursive $G_\D = \{f : \D \partialto \D \mid f - \text{general recursive}\}$ and total recursive, or computable, functions $\CC_\D = \{f : \D \to \D \mid f - \text{total recursive}\}$. For a general recursive function $f \in G_\D$, we use $D_f \subseteq \D$ to denote the set of inputs on which $f$ is defined.

\subsection{Computational Models} \label{subsec:CM}

We begin with some unconventional terminology, particular to our paper. The class of functions $G_\D$ emerges from the interaction between abstract machines and symbolic environments that house the input-output domain $\D$. We define an \textbf{abstract machine} as a computable dynamical system with a deterministic discrete-time evolution. When one supplies an abstract machine with a formalized input-setting and output-reading convention, then the resulting system becomes a \textbf{computational model} on $\D$. Let $\M_\D$ denote the set of models computing the general recursive functions in $G_\D$. For a model $M \in \M_\D$, we use $D_M \subseteq \D$ to denote the set of inputs for which $M$ halts.

We start our discussion of computational models with a formal introduction of finite-state transducers (FSTs) and Turing machines (TMs). Then, we muse a bit over the (Extended) Church-Turing Thesis and uncover a paradox that to the best of our knowledge appears unaddressed by the computability literature. We conclude the section with a formal description of computational agents.

\subsubsection{Finite-State Transducers} 
FSTs, also known as Mealy machines, serve as a fundamental models of computation, representing sequential decision processes with finite memory. FSTs map input sequences to output sequences via state transitions and play a key role in learning theory and formal language theory.
\begin{definition}[Finite-State Transducers] \label{def:FST}
    An FST $M$ is a 6-tuple $M = (Q, A, B, \d, \g, q_0)$ where $Q$ is a finite set of states, $A$ is a finite input alphabet, $B$ is a finite output alphabet, $\d : Q \times A \to Q$ is a transition function, $\g : Q \times A \to B$ is an output function, and $q_0 \in Q$ is the start state.
\end{definition}
To extend $\d$ and $\g$ to longer inputs, we recursively define
$\d(q, a_1 \dots a_{n+1}) \coloneqq \d(\d(q, a_1 \dots a_n), a_{n+1})$, $\g(q, a_1 \dots a_{n+1}) \coloneqq \g(\d(q, a_1 \dots a_n), a_{n+1})$
for all $n \geq 1$ and $a_1 \dots a_{n+1} \in A^{n+1}$. We define the \textbf{semantics} of $M$ as the function $\g_M = \g(q_0, \cdot) : A^+ \to B$, and we use $M : A^+ \to B^+$ for the seq2seq map $M(a_1, \dots, a_n) = (\g_M(a_1), \dots, \g_M(a_1, \dots, a_n))$. For a function $\chi : A^+ \to B$, we define the \textbf{Generalized Nerode Equivalence} as $u \equiv_\chi v \iff \chi(u \cdot w) = \chi(v \cdot w), \ \forall w \in A^+$.\footnote{We use $\cdot$ to denote concatenation. We invite the reader to \Cref{app:sec:fst} for additional discussion on FSTs.} We let $[u]_\chi$ denote the equivalence class of $u \in A^\star$, and we use $|\chi|$ for the number of equivalence classes. If $|\chi| < \infty$, we call $\chi$ a \textbf{rational map}. Since $|\g_M| \leq |Q|$, an FST semantics always represents a rational map. Surprisingly, the converse also holds, as per the generalized Myhill-Nerode theorem.

We denote the set of FSTs over I/O alphabets $A/B$ by $\Phi_A^B$ and the set of rational functions by $\mathcal{P}_A^B$. The set $\wt{\Phi}_A^B$ stands for partial FSTs with missing transitions, with the corresponding set of partial rational functions -- $\wt{\mathcal{P}}_A^B$. Following \citet{oncina1992inferring}, we define quotient transducers, with merged states according to a partition $\pi$ of $Q$.
\begin{definition}[Quotient Transducers] \label{def:quotient}
    Let $M = (Q, A, B, \d, \g, q_0) \in \wt{\Phi}_A^B$ and let $\pi$ be a partitioning of the state-space $Q$. We denote by $B(q, \pi) \subseteq Q(M)$ the unique cluster of the partitioning containing $q \in Q$. Now, we define the quotient (possibly non-deterministic) FST $M / \pi = (Q', A, B, \d', \g', q_0')$ as follows: $Q' = \{B(q, \pi) : q \in Q\}$; $(\d', \g')(B, a) = \{ (B', b) : \exists q \in B, q' \in B' \text{ s.t. } (\d, \g)(q, a) = (q', b)\}$, $\forall B \in Q', a \in A$; $q_0' = B(q_0, \pi)$. 
\end{definition}
To differentiate states, we use the notion of \textbf{state apartness}, following \citet{vaandrager2022new}. Two states $p, q \in Q$ are \textbf{apart}, denoted $p \# q$, if there exists a \textbf{distinguishing string} $\s \in A^+$ such that $\g(p, \s) \neq \g(q, \s)$. Hopcroft's algorithm \citep{hopcroft1971n} efficiently determines state apartness.

\subsubsection{Turing Machines}
One can define Turing machines with various architectural choices, including different numbers of heads and tapes, input-output conventions, halting conditions, and alphabet sizes. While all such variants prove computationally equivalent \citep[Chapter 1]{arora2009computational}, some provide greater convenience for exposition. We adopt a specific TM architecture that best suits our discussion.

Let $\l$ denote the blank symbol. Given a finite problem alphabet $\Sigma$ s.t. $\l \notin \Sigma$ and a finite tape alphabet $\G \supseteq \Sigma \cup \{\l\}$, our model family $\T$ computes all general recursive functions in $G_\SS$.

\begin{definition}[Turing Machine] \label{def:TM}
    A TM $T \in \mathrm{TM}^\Gamma$ is a triple $T = (Q, \G, \d)$ where $Q$ is a finite set of states, $\G$ is a finite tape alphabet, and $\d : Q \times \G \to Q \times \G \times \{L, R ,S\}$ is a transition function.
\end{definition}

\textit{TM Computation.}
In the above definition, the symbols $L, R,$ and $S$ dictate whether the TM head moves left, right, or stays in place after transitioning to a new state and writing a symbol from $\G$ over the current cell of a bidirectionally infinite tape. The computation of any $T \in \mathrm{TM}^\Gamma$ proceeds as follows: The input $x \in \SS$ is written on an otherwise-blank tape with the TM head positioned at the first symbol of $x$ and $T$ initialized at a specified initial state $q_0 \in Q$. The TM then follows the transition function $\d$. Computation halts when $T$ enters a designated halt state $q_f \in Q$. The output $T(x)$ is defined if and only if just a single contiguous string from $\SS$ occupies the tape in the moment of halting. Otherwise, $T(x)$ remains undefined. Now, the class of abstract machines $\mathrm{TM}^\Gamma$, together with this computation convention, specifies the model class $\T$ over the enumerable domain $\SS$. We overload the definition of a TM to refer both to an abstract machine and a computational model.

\textit{Representation Complexity.} 
We define $|T|$ as the number of states in $T \in \T$. Each TM uses exactly $|T| \cdot |\G|$ transitions. Since we keep the tape alphabet $\G$ constant, $O(|T| \log |T|)$ bits suffice for encoding $T$. Thus, $|T|$ serves as a meaningful proxy for the representational complexity of $T$.

\subsubsection{Encodings and Simulations}

The Church-Turing Thesis (CTT) effectively defines computability through Turing-computability by claiming that TMs can simulate any intuitively computable function modulo a \textit{reasonable} encoding. However, the term ``reasonable'' appears to rely on the notion of ``computable'' leading to an infinite regress. To resolve this paradox, we propose grounding reasonability in first-order logic (FOL). Thus, we define an encoding as reasonable if and only if the encoding derives from a TM-computable transformation of an inherent FOL representation of the domain (see details in \Cref{app:sec:ctt}).

\paragraph{The Extended Church-Turing Thesis}
As computability theory evolved, considerations extended beyond feasibility to the \textit{efficiency} of computation. Accordingly, the CTT received a strengthening in the form of the \textbf{Extended Church-Turing Thesis} (ECTT)\footnote{See \citep{sep-church-turing} for a detailed discussion.}, which asserts that the class of \textit{efficiently computable problems} is model-independent. Specifically, for any physically realizable computational model $M$, there exists a constant $c \in \N$ such that if $M$ evolves for $t$ steps, a TM $T \in \T$ can simulate $M$'s evolution within $t^c$ steps \citep[Chapter 1.5.2]{arora2009computational}. While less universally accepted, the ECTT remains a widely used conjecture. We adopt the following relaxed formulation.
\begin{assumption}[$q$-Extended Church-Turing Thesis] \label{ass:q-ECTT}
    Let $q : \N \times \N \to \N$ be a computable monotonically increasing ``overhead'' function. The $q$-ECTT states that for any physically realizable computational model $M$, there exists a constant $c \in \N$ and a TM $T \in \T$ such that if a computation on $M$ takes $t$ steps, then $T$ simulates the same computation in $q(c, t)$ steps.
\end{assumption}
For a computational model $M$ over a domain $\D$, the $q$-ECTT implies the existence of a reasonable encoding $\varphi : \D \to \SS$ such that some $T \in \T$ achieves the specified time-bound.
Since switching between non-adversarially constructed encodings introduces at most a constant slowdown to the simulation time\footnote{The TM $(\phi \circ \varphi^{-1}) \circ T \circ (\varphi \circ \phi^{-1})$ simulates $M$ with at most a constant slowdown for naturally constructed $\phi$.}, we take as granted that the $q$-ECTT implies the additional assumption that all natural encodings of $\D$ into $\SS$ cause the $q$-ECTT to kick in.

\subsubsection{Computational Agents}

Our \textit{IOO--Learning} (\Cref{sec:blackbox}) and \textit{TBO--Learning} (\Cref{sec:clock}) results apply to the general unrestricted model family $\M_\D$ over the domain $\D$. However, for \textit{PTO--Learning} (\Cref{sec:CoT}), we introduce the concept of a computational agent with observable behavior.

Following \citet{russell2016artificial}, we view a \textbf{computational agent} as a system that perceives and acts upon an environment. We model this environment as a discrete symbolic universe $\U$ for two key reasons. First, upon introspection, any reasoning or mechanical task computed by humans appears to involve manipulation of symbols in a discrete space. Second, though physical reality possibly incorporates both discrete and continuous quantities, the unavoidable noise in physical observations will cause the sensory data of any computational system to achieve only finite precision.

\begin{definition}[Symbolic Universe] \label{def:universe}
    Given an enumerable set $G$ and a finite set of symbols $\G$ containing the empty symbol $\l$, a world-state $w : G \to \G$ is a function with finitely many non-empty assignments. The countable symbolic universe $\U = \U(G, \G)$ is the set of all such world-states.
\end{definition}
The set $G$ defines the geometry of the universe. For example, for TMs, $G = \Z$ represents tape positions, and world-states $w : \Z \to \mathcal{S}$ correspond to tape configurations. 
We now define a computational agent as an entity interacting with a symbolic universe. Let $\mathcal{P}(\cdot)$ denote powersets.

\begin{definition}[Computational Agent] \label{def:agent}
    A computational agent operating in a symbolic universe $\U(G, \G)$ constitutes a triple $M = (Q, \U, \d)$, where $Q$ is a finite set of states and $\d : Q \times \mathcal{P}(G) \times \U \to Q \times \mathcal{P}(G) \times \U$ is a computable transition function with perception restrictions given below.
\end{definition}
The transition function $\d$ determines how the agent evolves:
$\d(q, P, w) = (q', P', w')$, where $q \in Q$ is the current state, $P \subseteq G$ is the agent’s perception (the observed region of the universe), and $w \in \U$ is the current world-state. The agent updates the internal state ($q'$), shifts perception ($P'$), and modifies the environment ($w'$). Notably, $\d$ can only access $w_{|P}$, i.e., the portion of the world within the field of perception. For a complete interaction trace, we specify an initial state $q_0 \in Q$, an initial perceived region $P_0 \subseteq G$, and a final state $q_f \in Q$ as a halting condition. Since we model the environment as static and the agent follows a fixed transition function, the system defines a \textbf{symbolic dynamical process} $\d : \Lambda \to \Lambda$, where $\Lambda = Q \times \mathcal{P}(G) \times \U$.

\textit{Examples of Computational Agents.}  
Many computational models fit this framework. 
TMs, RNNs, and transformers all operate on symbolic universes with geometry $G = \Z$. For these models, the symbol-space $\G$ corresponds to the tape alphabet, input-output symbols, or token set. TMs serve as the prototypical \textit{universal} computational agents, with a continuously moving singleton perception. RNNs and transformers (under realistic constraints such as finite-precision weights and finite context windows) reduce to complex finite-state transducers. Humans may also function as computational agents when performing structured tasks without modifying their neural architecture.

\subsection{Algorithmic Learning Theory} \label{subsec:ALT}

We formalize the learning-in-the-limit framework, focusing on observations and characteristic sets.

\subsubsection{Learning in the Limit with Input Sources and $\a$-Observations} 
We extend Gold's learning-in-the-limit framework by incorporating a family of computational models, a restricted input source, and a broader notion of natural observations.

Our generalization introduces the set of \textbf{computational models} $\M_\D$ over an enumerable domain $\D$ which compute the general recursive functions in $G_\D$. For any subset $\M \subseteq \M_\D$, let $\F_\M \subseteq G_\D$ denote the set of functions computed by $\M$, and for any $M \in \M_\D$, let $f_M : \D \partialto \D$ denote the function computed by $M$. We aim to learn the \textbf{concept class} $\F_\M \subseteq G_\D$.

\textit{Information Presentation.} Given a ground-truth model $M \in \M$ from a \textbf{model class} $\M \subseteq \M_\D$, the learning algorithm $\L$ observes examples from an \textbf{input source} $I \subseteq \D$ such that \textbf{inputs} $x$ only arrive from the set $I_M \coloneqq I \cap D_M$ and produce observable \textbf{examples} $\eta_M(x) = (x, f_M(x), \a(M, x))$ where $\a(M, x)$ contains \textbf{$\a$-observations} about the computation of $f_M(x)$. For a sample set $S \subseteq I_M$, we denote the observed examples by $E_M(S) = \{\eta_M(x) : x \in S\}$. We define the equivalence class $[M]_S = \{T \in \M : E_T(S) = E_M(S)\}$ of indistinguishable models on $S$. Now, an adversary orders the examples using a \textbf{surjective ordering} $w : \N \surj I_M$, so that at time $t$, the learner receives $\eta_M(w_t)$. The learner produces a hypothesis $R_t = \L(\eta_M(w_1), \dots, \eta_M(w_t))$, where $R_t$ belongs to a computable \textbf{set of representations} $\RR$. Through a \textbf{simulation function} $\mathfrak{S} : \RR \to G_\D$, the representation $R_t$ generates a general recursive\footnote{Since $\CC_\D$ is not recursively enumerable, $\RR$ might include representations of some non-total recursive functions.} hypothesis $\mathfrak{S}(R_t)$.
\begin{definition}[The $(\M, \a)$--LP] \label{def:LP}
    For a model class $\M \subseteq \M_\D$, we say that a learning algorithm $\L$ solves the $(\M, \a)$--learning problem\footnote{Equivalently, $\L$ learns $\F(\M)$ in the limit from $\a$-observations of $\M$ and restricted inputs.} if for all models $M \in \M$, all input sources $I_M \subseteq D_M$, and all surjective orderings $w : \N \surj I_M $ of $E_M(I_M)$ from $\a$-observation, after some finite time $t^\star = t^\star(M, I, w)$, $\L$ converges to a correct representation $\hat{R} = R_t, \ \forall t \geq t^\star$ s.t. $\mathfrak{S}(\hat{R})_{|I_M} = f_{M|I_M}$.
\end{definition}
The $(\M, \a)$--LP framework generalizes Gold's learning-in-the-limit setup as follows. \textbf{(a)} Introduction of computational models and $\a$-observations: Learning now depends on an unknown model $M$ computing $f_M$ whose computation provides additional structured information. \textbf{(b)} Restriction on input sources: Observations may now come only from a subset $I$ of the full input domain $\D$.

Notice that if $\M \subseteq \M'$, then a learner $\L$ solving the $(\M', \a)$--LP will also solve the $(\M, \a)$--LP. Hence, the hardest learning-in-the-limit setting becomes the unrestricted $(\M_\D, \a)$--LP where the learner lacks prior knowledge about the ground-truth process that could reduce the complexity of the hypothesis space. Now, the set of all possible models $\M_\D$ acting on $\D$ includes a multitude of different computational species, which makes the existence of a unified learning strategy all the more surprising. The motivation for this learning setup comes from envisioning a robot trying to learn a computational task from observations. The robot observes and understands the domain of action $\D$ but lacks information about the architectural specifics of the black-box model performing the task. Moreover, only a restricted subset $I$ of the domain $\D$ might produce relevant inputs to the target task.

\subsubsection{Characteristic Sets}
The $(\M, \a)-$LP provides a formal criterion which guarantees genuine learning robust to example reorderings and input-source restrictions. However, once a learning algorithm $\L$ solves the $(\M, \a)-$LP, we would also like to quantify $\L$'s data efficiency under favorable curricula. Formalized by \citet{de1997characteristic}, the characteristic-sets framework provides such a characterization. We adapt the original definition to incorporate $\a$-observations and input restrictions.

\begin{definition}[Characteristic Sets] \label{def:characteristic}
    Let $\M \subseteq \M_\D$ and let $\L$ solve the $(\M, \a)$--LP. Characteristic sets exist relative to $\L$, $\M$, $\a$, and $I \subseteq \D$. We define $S_M \subseteq I \cap D_M$ as a characteristic set for $M \in \M$ if $\L(E_M(S))$ computes the same function as $M$ restricted to $I \cap D_M$ whenever $S \supseteq S_M$.\footnote{Notice that an effective curriculum for learning the function of $M$ with $\L$ and $\a$ would prioritize inputs from $S_M$.}
\end{definition}

Intuitively, characteristic sets form a core set of inputs that enable efficient model identification. We measure the complexity of $S_M$ in two key ways. First, we define the \textbf{mass} of $S_M$ as
$\n{S_M} = \sum_{s \in S_M} \texttt{length}(x, f(x), \a(M, x))$,
which quantifies the  size of the data $\L$ must process. Second, we define the \textbf{size} of $S_M$ as $|S_M|$. Then, a model class $\M$ is \textbf{$\a$-restrictively identifiable in polynomial time and data} ($\a$-RIPTD) if there exists a polynomial-time learner $\L$ for which every $M \in \M$ admits a characteristic set with polynomial mass $\n{S_M} = \texttt{poly}(|M|)$ regardless of the input source $I$. For $\a$-RIPTD model classes, the complexity of representations determines the hardness of learning. \citet{gold1978complexity} showed that the FST model class $\Phi_A^B$ is $\eps$-IPTD (only input-output observations and fixed $I = A^+$). In contrast, in \Cref{sec:CoT}, we prove that $\Phi_A^B$ is not $\eps$-RIPTD. Thus, requiring identifiability for a variable input source makes learning considerably harder.

Note that the mass and size of characteristic sets can only increase when expanding the model class $\M$ or reducing the informational content of $\a$. In particular, if $\M$ is not $\a$-RIPTD, then any $\M' \supseteq \M$ is also not $\a$-RIPTD. We conclude with the following observation proved in \Cref{app:sec:c-sets}.

\begin{restatable}[Distinguishability]{lemma}{distinguish} \label{lem:distinguish}
    Let $\L$ solve the $(\M, \a)$--LP. Suppose that characteristic sets exist relative to $\L, \M, \a, \text{ and } I \subseteq \D$. Let $M, M' \in \M$ compute two distinct functions when restricted to $I$. Then, there exists $x \in S_M \cup S_{M'}$ such that either $x \notin D_M \cap D_{M'}$ or $\eta_M(x) \neq \eta_{M'}(x)$.
\end{restatable}

We mention in passing that learning efficiency in other frameworks implies the existence of learning-in-the-limit algorithms with small characteristic sets. For instance, a valid $T/L$ pair from \citep{goldman1993teaching} leads to the existence of computable characteristic sets from model representations, and semi-polynomial $T/L$-teachability implies IPTD (see \citep{de1997characteristic}). Moreover, polynomial-mass/size example-based algorithms (like Angluin's $L^\star$ algorithm \citep{angluin1987learning}) lead to algorithms with polynomial-mass/size characteristic sets due to \citep[Proposition 1]{de1997characteristic} + \citep[Theorem 2]{goldman1993teaching}.





\section{Learning from Input-Output Observations} \label{sec:blackbox}

In this original formulation of learning in the limit, the learner $\L$ observes only inputs and outputs, with no additional information about the computation process. Thus, $\a(\cdot, \cdot) \equiv \eps$, and the model family $\M$ contributes nothing beyond function evaluations. We refer to this setting as the $\mathbf{\F_\M}$\textbf{--LP}.

For the input-output learning setup, \citet{gold1967language} proved the following negative result through a diagonalization argument that forces any fixed learner to switch between hypotheses infinitely often.

\begin{theorem}[Gold, Theorem~I.5]
    No algorithm learns the class of computable functions in the limit from input-output observations. That is, for a finite alphabet $\Sigma$, no algorithm solves the $\CC_\SS$--LP.
\end{theorem}

Interestingly, we prove that restricting the time complexity of the concept class enables learnability in the limit from input-output observations and arbitrary input sources. Now, for an enumerable domain $\D$, let $\texttt{size}_\D : \D \to \N$ represent some canonical measure of the complexity of the input known to the learner $\L$. We introduce the following parametrized complexity classes.
\begin{definition}[General Complexity Classes] \label{def:complexity}
    For a model class $\M \subseteq \M_\D$ over an enumerable set $\D$ and for a computable monotonically increasing function $Q : \N \times \N \to \N$, let $\mathbf{TIME}_\M(Q(c, n))$ denote the subset of general recursive functions $f \in \F_\M$ for which there exists a computational model $M \in \M$ -- not necessarily a TM -- that computes $f$ on halting inputs $x \in \D_f$ of $\texttt{size}_\D(x) = n$ in $O(Q(c, n))$ steps. We define the following complexity class $\mathbf{Q}(\M) = \bigcup_{c \in \N} \mathbf{TIME}_\M(Q(c, n))$.
\end{definition}
Now, under a relaxed form of the Extended Church-Turing Thesis, we prove that the concept class $\mathbf{Q}(\M_\D)$ is learnable in the limit from input-output observations and arbitrary input sources.

\begin{restatable}[Time-Restricted IOO-Learning]{theorem}{blackbox} \label{thm:blackbox}
    Assuming the $q$-ECTT (\ref{ass:q-ECTT}) holds, for an enumerable set $\D$ and a computable monotonically increasing function $Q : \N \times \N \to \N$, there exists an algorithm which learns the $Q-$time-bounded general recursive functions over $\D$ from input-output observations and restricted inputs. Formally, there exists an algorithm solving the $\mathbf{Q}(\M_\D)$--LP.
\end{restatable}

The proof of \Cref{thm:blackbox}, delegated to \Cref{app:blackbox}, provides a learning-by-enumeration strategy which outputs the min-state TM simulating the ground truth. Since the number of $n-$state TMs in $\T$ is super-exponential in $n$, such an enumeration algorithm requires super-exponential effort in the size of the ground-truth representation to converge. 

Now, even if the $q$-ECTT fails, we can still salvage an important corollary if we stick to complexity classes over Turing machines: i.e., $\mathbf{Q}(\T)$. Then, the proof of \Cref{thm:blackbox} readily extends to this setting by simply replacing the encoding $\varphi$ with the identity and letting $q(c, n) = n$. Hence, popular complexity classes like $\mathbf{P}, \mathbf{NP}, \mathbf{\#P}, \mathbf{L},$ and $\mathbf{EXP}$ become learnable in the limit.

\begin{corollary}
    There exists an algorithm for learning any parametrized complexity class of general recursive functions in the limit from input-output observations and restricted inputs. Formally, for any finite problem alphabet $\Sigma$, tape alphabet $\G \supseteq \Sigma \sqcup \{\l\}$, and computable monotonically increasing function $Q : \N \times \N \to \N$, there exists an algorithm solving the $\mathbf{Q}(\T)$--LP.
\end{corollary}

\section{Learning from Time-Bound Observations.} \label{sec:clock}

Given a model class $\M \subseteq \M_\D$ over an enumerable domain of computation $\D$, time-bound observations provide the learner $\L$ not only with input-output information $(x, f_M(x))$ about the ground-truth model $M \in \M$ but also with an approximate upper bound $\atb(M,x) \in \N$ on the number of discrete computation steps taken by $M$ to compute $f_M(x)$. In other words, if we denote with $t_M(x)$ the number of computation steps $M$ takes on a halting computation trajectory, then there exists some constant dependent on $M$ such that $t_M(x) = O_M(\atb(M,x)), \ \forall x \in D_M$.

To motivate the practical relevance of such observations, we could consider a robot with restricted access to the computation of some model $M$. Perhaps, the robot only observes a terminal that once provided with an input $x$, returns $f_M(x)$ after a certain amount of time. In such information-poor scenarios, the robot could still reference its internal clock and measure the number of seconds $\atb(M, x)$ between entering $x$ in the terminal and observing $f_M(x)$. If each discrete computation step of $M$ takes $1/p$ seconds, then $M$ will compute $f_M(x)$ in at most $\ceil{p\atb(M,x)}$ steps.


In the following universal learning theorem, the $q$-ECTT allows us to translate the abstract computational models of $\M_\D$ into concrete TM representations. Furthermore, time-bound observations serve a similar role to that of restricting the complexity class of the target function as in our IOO-learning results.

\begin{restatable}[Universal TBO-Learning]{theorem}{clock}\label{thm:clock}
    Assuming the $q$-ECTT (\ref{ass:q-ECTT}) holds, for an enumerable set $\D$, there exists an algorithm for learning the class of general recursive functions $G_\D$ from time-bound observations of arbitrary computational models and input sources. Formally, there exists an algorithm solving the $(\M_\D, \atb)$--LP.
\end{restatable}

The proof of \Cref{thm:clock}, deferred to \Cref{app:clock}, mimics the learning-by-enumeration technique we developed for input-output learning. Again, the learner $\L$ which solves the $(\M_\D, \atb)$--LP returns a consistent min-state model and requires super-exponential runtime to converge. Now, even if the $q$-ECTT fails empirically, we can still prove that the model class of Turing machines $\T$ is learnable in the limit from time-bound observations. This corollary follows by replacing the encoding $\varphi$ with the identity function and setting $q(c, n) = n$ in the proof of \Cref{thm:clock}.

\begin{corollary}
    There exists an algorithm for learning the class of general recursive functions in the limit from time-bound observations of Turing machines on restricted inputs. Formally, there exists a learning algorithm solving the $(\T, \atb)$--LP.
\end{corollary}

\section{Learning from Policy-Trajectory Observations} \label{sec:CoT}

In the PTO framework, the learner $\L$ observes the full interaction of a computational agent $M=(Q, \U, \d)$ with its symbolic universe $\U(G, \G)$. As discussed in \Cref{subsec:CM}, this interaction forms a computable symbolic dynamical process $\d : \Lambda \to \Lambda$, where $\Lambda = Q \times \mathcal{P}(G) \times \U$. Importantly, the transition function $\d$ updates a triple $(q, P, w) \in \Lambda$ according to the current agential state $q$ and the information gleaned from the field of perception $w_{|P}$. We propose that the update $(P, w) \mapsto (P', w')$ should represent the observable part of the full interaction $(q, P, w) \mapsto (q', P', w')$. That is, we intend to decompose the interaction space $\Lambda$ into a hidden part $\Lambda_{H} = Q$ and an observable part $\Lambda_O = \mathcal{P}(G) \times \U$ such that $\Lambda = \Lambda_H \times \Lambda_O$.

The motivation for such a decomposition comes from human learning. For example, when a teacher ($M$) solves a quadratic equation ($f$), the student ($\L$) observes both the full blackboard ($w$) and the chalk interaction spots ($P$) as the teacher paces around and makes symbolic changes. More generally, humans seem quite capable of detecting where the perception of fellow humans (or other animals) falls, which allows us to make educated guesses about the inner lives of the observed. Of course, the closer the evolutionary distance to the observed (i.e., the more we instinctively understand about the agent family $\M$), the more educated the guess. Thus, knowledge of $\M$'s architectural specifics appears essential if we want to link external behavior to internal configurations. Accordingly, we will soon adopt the family of TMs $\T$ as our concrete object of study.

Now, since the learner $\L$ might lack the capacity to observe the full world-state $w$ (e.g., due to a very large blackboard) and since only the attended part of $w$ carries a learning signal, we further compress the observation $(P_t, w_t) \mapsto (P_{t+1}, w_{t+1})$ at time $t+1$ into the action tuple $\Delta_{t+1} = (w_{t|P_t}, w_{t+1|P_t}, D_{t+1})$, where $D_{t+1} = \text{``} P_{t+1} - P_t \text{''}$ denotes the shift in perception. Hence, we define a \textbf{policy-trajectory observation} for an agent $M$ as a sequence of action tuples $\tau = \{\Delta_t\}_{t=1}^L$.

\paragraph{Connection to Reinforcement Learning.} Here, we outline some superficial similarities between our PTO-Learning framework and Learning from Demonstration (LfD), also known as Apprenticeship Learning -- a subgenre of RL. For a longer discussion of LfD, we refer the reader to \citep{russell2016artificial, sutton2018reinforcement, correia2024survey}.

Now, the LfD framework often uses a POMDP as a foundation. The demonstration dataset $D = \{\tau_i\}_{i=1}^N$ consists of expert trajectories $\tau_i = (o_t^i, a_t^i)_{t=0}^{L_i}$ where $o_t^i$ belongs to a finite set of possible partial observations $\O$, $a_t^i$ belongs to a finite set of possible actions $\mathcal{A}$, and $L_i$ denotes the length of the interaction. Based on this demonstration data, the agent tries to develop a policy approximating the stochastic history-dependent expert policy $a_t^i \sim\pi^\star(o_0^i, \dots, o_t^i)$.

Similarly, for PTO-Learning, we can think of the agent $M$ as implementing a deterministic policy $\pi_M$ which takes as input a history of partial world-state observations $\mathfrak{h}_t=(w_{0|P_0}, \dots, w_{t|P_t}) \in \G^+$ and outputs an action pair consisting of a symbol $w_{t+1|P_t}$ and an attention shift $D_{t+1}$. If we let $\mathcal{S}$ denote the finite set of possible attention shifts, then the policy of $M$ becomes a partial function $\pi_M : \G^+ \partialto \G \times \mathcal{S}$. Interestingly, one could extend $\pi_M$ to a \textit{rational function} on the full domain $\G^+$ as we will see later on in the section. Now, if $M$ takes $t_M(x)$ steps to compute $f_M(x)$, then the policy-trajectory observation becomes
\begin{equation*}
    \apt(M, x) = (w_{t-1|P_{t-1}}, w_{t|P_t}, D_t)_{t=1}^{t_M(x)} = (w_{t-1|P_{t-1}}, \pi_M(\mathfrak{h}_{t-1}))_{t=1}^{t_M(x)} \in (\G^2 \times \mathcal{S})^{t_M(x)}.
\end{equation*}

\paragraph{TM Policy Trajectories.} 
We focus on policy-trajectory observations of TMs, making PTO-Learning a species-aware framework: The learner $\L$ knows $M \in \T$ and generates hypotheses using representations from $\T$. The additional information $\apt(M, x)$ consists of a sequence of tape manipulations
$\apt(M, x) = (\Delta_1, \dots, \Delta_{t_x}), \text{ where } \Delta_i = (\s_i, \s_i', D) \in \G^2 \times \{L, R, S\}$.
Each $\Delta_i$ records the TM’s $i^{\text{th}}$ operation: Reading $\s_i$, writing $\s_i'$, and moving $D$.

We proceed to introduce some useful notation. For a TM $T \in \T$, we denote by $T[x]$ and $T\{x\}$ the sequence of tape manipulations and the sequence of symbols scanned by $T$ on input $x$, respectively. If $T[x] = (\Delta_1, \dots, \Delta_{t_x})$ where $\Delta_i = (\s_i, \s_i', D) \in \G^2 \times \{L,R,S\}$, then $T\{x\} = (\s_1 \dots \s_{t_x}) \in \G^+$. Note that this notation extends to infinite sequences and that $T[x] = \apt(T, x)$. Furthermore, for a sample of inputs $S \subseteq \SS$, we overload the notation $T[S] = \{T[x] : x \in S\}$ and $T\{S\} = \{T\{x\} : x \in S\}$. We refer to $T[x]$ and $T[S]$ as \textbf{tape behavior}, and we denote with $[T]_S \subseteq \T$ the equivalence class of TMs with the same tape behavior on $S$ as $T$.

\paragraph{Connection to FSTs.} Interpreting $A = \G$ and $B = \G \times \{L, R, S\}$ as input/output alphabets, we define a mapping $\psi: \T \to \Phi_A^B$ that converts any TM $T \in \T$ into an FST $\psi(T) \in \Phi_A^B$ with the same transition diagram. Consequently, tape behavior $T[x]$ corresponds to a sequence of input-output observations $(T\{x\}, \psi(T)(T\{x\}))$ of the rational function $\g_{\psi(T)}$ corresponding to the semantics of $\psi(T)$. This reduces the $(\T, \apt)$--LP to learning the class of rational functions $\mathcal{P}_A^B$ in the limit from IOOs with a restricted source: i.e., the $\mathcal{P}_A^B$--LP.\footnote{This change requires extending Definition \ref{def:LP} to models with different input/output domains.} Notice that if $T$'s inputs come from $I \subseteq \SS$, then $\psi(T)$'s inputs come from $M\{I\} \subseteq \G^+$.  We summarize these findings in the following lemma. 
\begin{restatable}[Recursive-to-Rational Reduction]{theorem}{reduction} \label{thm:reduction}
    Learning the class of general recursive functions in the limit from PTOs reduces to learning the class of rational functions in the limit from IOOs. Formally, an algorithm solving the $\mathcal{P}_A^B$--LP with $A = \G$, $B = \G \times \{L,R,S\}$ will solve the $(\T, \apt)$--LP. In particular, there exists a learning-by-enumeration algorithm for the $\mathcal{P}_A^B$--LP.
\end{restatable}

The proof of \Cref{thm:reduction}, deferred to \Cref{app:cot}, shows the learnability of the $(\T, \apt)$--LP.

\paragraph{Uncomputable Characteristic Sets.}
We now state a negative result showing that no algorithm can learn the $(\T, \apt)$--LP or the $\mathcal{P}_A^B$--LP with bounded-mass characteristic sets. In particular, our theorem proves that input-source restrictions make learning much harder since for unrestricted inputs, the rational functions class $\mathcal{P}_A^B$ is IPTD \citep[Theorem 4]{gold1978complexity}.

\begin{restatable}[Unbounded Characteristic Sets]{theorem}{unbounded} \label{thm:unbounded}
    Fix a computable bounding function $\b : \T \to \N$ and let $I = \SS$. Then, for $\apt$-observations over the input source $I$, no algorithm achieves characteristic sets for $\T$ with $\b$-bounded mass. Furthermore, for $\apt$-observations over $I$, no algorithm for learning the class $\T$ with computable characteristic sets exists.
\end{restatable}

The proof of \Cref{thm:unbounded}, found in \Cref{app:cot}, uses a reduction to the Halting Problem and shows that even the model class of linear-time TMs cannot have computable characteristic sets.\footnote{Thus, no valid $T/L$ pair exists for the class of linear-time computable functions under policy-trajectory observations.} Hence, we strongly conclude that the representation complexity of TMs does not control the hardness of learning $G_\SS$ in the limit from policy-trajectory observations of $\T$. In light of the proof of \Cref{thm:unbounded}, one might argue that learning efficiency should be measured not by the \textit{mass} of characteristic sets $\n{S_M}$ but by their \textit{size} $|S_M|$ (see \Cref{app:cot}).\footnote{This consideration is a long-standing open question in grammatical inference. See \citep[Problem 1]{de2006ten}.}

As a small corollary, setting $\b$ as any polynomial implies that $\T$ is not $\apt$-IPTD. Consequently, $\T$ is also not IPTD from time-bound or input-output observations.
\begin{corollary}
    $\T$ is not IPTD from input-output, time-bound, or policy-trajectory observations.
\end{corollary}

\paragraph{Observation Trees.} 
As shown in \Cref{thm:reduction}, one can solve the $\mathcal{P}_A^B$--LP, and thus the $\T$-LP, via enumeration. However, this approach requires an super-exponential runtime in the size of the ground-truth representation in order to eliminate all hypotheses preceding the correct one. Moreover, learning-by-enumeration fails to leverage structural information revealed by the observations. In contrast, both active \citep{angluin1987learning, pitt1989inductive, vaandrager2022new} and passive \citep{gold1978complexity, dupont1996incremental, parekh2001learning} FST-learning algorithms exploit structural information by constructing observation structures from example sets $E_M(S)$ and applying state-merging strategies. Notably, these algorithms assume an unrestricted input source $I = A^\star$, allowing free exploration of the state-transition diagram. For arbitrary input sources, learning becomes significantly harder -- at least as difficult as learning TMs from behavior observations, as \Cref{thm:reduction} suggests. The key challenge in learning TMs stems from the fact that, unlike FSTs, their transition function dictates which cells to scan, restricting direct exploration of the state-transition diagram. As a result, we enter uncharted territory, where we aim to develop a fast algorithm for learning FSTs from arbitrary input sources. To incorporate structural information from examples, we introduce the following definition.

\begin{definition}[Observation Tree] \label{def:tree}
    Given an example set $E_M(S) = \{(x, M(x)) : x \in S\}$ for an automaton $M \in \Phi_A^B$, we define the observation tree $\mathcal{T}_M(S) \in \wt{\Phi}_A^B$ as a partial FST with paths from the root to leaf states which correspond exactly to the input-output sequences from $E_M(S)$.
\end{definition}

\paragraph{State Merging.} Given an input source $I \subseteq A^\star$ and an automaton $M \in \Phi_A^B$, we denote by $I(M) \in \wt{\Phi}_A^B$ the partial automaton obtained after the pruning of unused states and transitions when all inputs come from $I$. We want to merge the states of $\mathcal{T}_M(S)$ into clusters such that the transformed automaton computes the same partial rational function as $I(M)$. When we merge two states $p, q \in Q(\mathcal{T}_M(S))$, we also merge all of the transition paths going through $p$ and $q$. Hence, the merger of two states could potentially cascade into further mergers. Moreover, mergers can only occur between states with non-contradictory paths, hence the following definition.
\begin{definition}[Valid Merger] \label{def:merger}
    Let $\pi$ denote the merging of the state-space $Q(\mathcal{T}_M(S))$ into the clusters $C_1, \dots, C_k$. We define $\pi$ as a valid merger if the quotient automaton $\mathcal{T}_M(S) / \pi$ is deterministic.
\end{definition}
If $|Q(\mathcal{T}_M(S))| = n$, then one can test the validity of a certain merger in $O(n|A|)$ time by checking the transition function of $\mathcal{T}_M(S) / \pi$. Also, observe that one can merge all the paths containing $p, q \in Q(\mathcal{T}_M(S))$ in $O(n|A|)$ time by performing mergers in a BFS manner (starting from the merger of $p$ and $q$) and noting that $(n-1)$ upper-bounds the possible total number of mergers. From now on, we will treat $|A|$ as a constant. Let us define the \textbf{similarity score} $s(p,q)$ as 0, in case $\texttt{merge}(p,q)$ produces an invalid merging, and as $r$, in case $\texttt{merge}(p,q)$ produces a deterministic quotient automaton $\mathcal{T}_M(S) / \pi$ with $n-r$ states. In other words, $s(p,q)$ measures the similarity of the transition paths going through $p$ and $q$. Now, since we can both perform $\texttt{merge}(p,q)$ and test the validity of that merger in linear time, we can also compute the similarity score $s(p, q)$ in $O(n)$ time.

The state-merging algorithms of \citet{oncina1992inferring, vaandrager2022new}, which solve the $\mathcal{P}_A^B$--LP under an unrestricted input source $I = A^\star$, attempt to incrementally construct a sub-tree automaton of $\mathcal{T}_M(S)$ by merging indistinguishable states. This approach works when the input is unrestricted, as every state in the ground truth automaton is eventually observed with distinguishing strings. With a restricted input source, however, most states in $\mathcal{T}_M(S)$ remain indistinguishable, making naive merging unreliable and prone to errors. 
Instead, we propose the \textbf{Maximum-Similarity Merging} (MSM) algorithm, which merges states only when sufficient evidence supports the merger. Given a partial automaton $M \in \wt{\Phi}_A^B$, MSM iteratively merges the state pair $(p,q)$ with the highest nonzero similarity score $s(p,q)$. The process continues until all similarity scores drop to zero, at which point MSM returns the generated quotient automaton. The naive implementation of MSM runs in $O(n^4)$ time: Finding the highest similarity score among $n$ states requires $O(n^3)$ time, and at most $n - 1$ mergers occur.

Unfortunately, the greedy MSM strategy provably cannot learn $\mathcal{P}_A^B$ in the limit, even under unrestricted input observations. However, we conjecture that for many natural orderings $w : \N \surj I$ of the input set, MSM successfully learns a correct FST representation in the limit.  We leave as an open question the characterization of the set of favorable orderings: $W_M^I = \{w : \N \surj I \mid \text{MSM learns } M \in \Phi_A^B \text{ in the limit} \}$. Finally, we show that MSM enjoys wide applicability. Namely, MSM can learn the class of all recursive functions in the limit when policy-trajectory observations originate from a restricted class of TMs. The proof of this result appears in \Cref{app:cot}.
\begin{restatable}{theorem}{MSM} \label{thm:MSM}
    For every $f \in \CC_\SS$, there exists $\G_f \supset \Sigma$ and a TM $T_f \in \mathrm{T}_\Sigma^{\Gamma_f}$ such that MSM learns $f$ in the limit from policy-trajectory observations of $T_f$ and from a polynomial number of samples .
\end{restatable}

\section{Conclusion} \label{sec:conclusion}

In this paper, we extend the learning-in-the-limit framework to include observations of the computational process under a restricted input source. By formally modeling different types of computational information—from simple runtime estimates to full behavioral trajectories—this work establishes new theoretical bounds on the limits of learning algorithms and offers insights into how intelligent agents can learn complex computational tasks from observing the process, not just the outcome. While classical input-output observations prevent learning the class of computable functions in the limit, we prove that both rough runtime estimates and policy-trajectory observations remove this learning barrier. Furthermore, by reducing the learning of general recursive functions with policy-trajectory observations to the learning of rational functions with input-output observations, we open the way for the application of FSA-identification algorithms to the more general problem of learning Turing machines. Open questions remain regarding the existence of polynomial-size characteristic sets for learning rational functions under restricted input observations and the impact of input orderings on the correctness of greedy state-merging algorithms. Future research should explore broader computational paradigms, including agents interacting with reactive environments.

\acks{This work was partially funded by an unrestricted gift from Google and by the Swiss National Science Foundation (grant number 212111).}

\bibliography{references}

\begin{thebibliography}{31}
\providecommand{\natexlab}[1]{#1}
\providecommand{\url}[1]{\texttt{#1}}
\expandafter\ifx\csname urlstyle\endcsname\relax
  \providecommand{\doi}[1]{doi: #1}\else
  \providecommand{\doi}{doi: \begingroup \urlstyle{rm}\Url}\fi

\bibitem[Aaronson(2005)]{aaronson2005guest}
Scott Aaronson.
\newblock Guest column: Np-complete problems and physical reality.
\newblock \emph{ACM Sigact News}, 36\penalty0 (1):\penalty0 30--52, 2005.

\bibitem[Angluin(1987)]{angluin1987learning}
Dana Angluin.
\newblock Learning regular sets from queries and counterexamples.
\newblock \emph{Information and computation}, 75\penalty0 (2):\penalty0 87--106, 1987.

\bibitem[Anil et~al.(2022)Anil, Wu, Andreassen, Lewkowycz, Misra, Ramasesh, Slone, Gur-Ari, Dyer, and Neyshabur]{anil2022exploring}
Cem Anil, Yuhuai Wu, Anders Andreassen, Aitor Lewkowycz, Vedant Misra, Vinay Ramasesh, Ambrose Slone, Guy Gur-Ari, Ethan Dyer, and Behnam Neyshabur.
\newblock Exploring length generalization in large language models.
\newblock \emph{Advances in Neural Information Processing Systems}, 35:\penalty0 38546--38556, 2022.

\bibitem[Arora and Barak(2009)]{arora2009computational}
Sanjeev Arora and Boaz Barak.
\newblock \emph{Computational complexity: a modern approach}.
\newblock Cambridge University Press, 2009.

\bibitem[Church(1936)]{church1936note}
Alonzo Church.
\newblock A note on the entscheidungsproblem.
\newblock \emph{The journal of symbolic logic}, 1\penalty0 (1):\penalty0 40--41, 1936.

\bibitem[Copeland(2023{\natexlab{a}})]{sep-church-turing}
B.~Jack Copeland.
\newblock {The Church-Turing Thesis}.
\newblock In Edward~N. Zalta and Uri Nodelman, editors, \emph{The {Stanford} Encyclopedia of Philosophy}. Metaphysics Research Lab, Stanford University, 2023{\natexlab{a}}.

\bibitem[Copeland(2023{\natexlab{b}})]{sep-entscheidungsproblem}
B.~Jack Copeland.
\newblock The rise and fall of the entscheidungsproblem.
\newblock In Edward~N. Zalta and Uri Nodelman, editors, \emph{The {Stanford} Encyclopedia of Philosophy}. Metaphysics Research Lab, Stanford University, 2023{\natexlab{b}}.

\bibitem[Correia and Alexandre(2024)]{correia2024survey}
Andr{\'e} Correia and Lu{\'\i}s~A Alexandre.
\newblock A survey of demonstration learning.
\newblock \emph{Robotics and Autonomous Systems}, 182:\penalty0 104812, 2024.

\bibitem[De~La~Higuera(1997)]{de1997characteristic}
Colin De~La~Higuera.
\newblock Characteristic sets for polynomial grammatical inference.
\newblock \emph{Machine Learning}, 27:\penalty0 125--138, 1997.

\bibitem[de~la Higuera(2006)]{de2006ten}
Colin de~la Higuera.
\newblock Ten open problems in grammatical inference.
\newblock In \emph{International Colloquium on Grammatical Inference}, pages 32--44. Springer, 2006.

\bibitem[Dupont(1996)]{dupont1996incremental}
Pierre Dupont.
\newblock Incremental regular inference.
\newblock In \emph{International Colloquium on Grammatical Inference}, pages 222--237. Springer, 1996.

\bibitem[Gold(1967)]{gold1967language}
E~Mark Gold.
\newblock Language identification in the limit.
\newblock \emph{Information and control}, 10\penalty0 (5):\penalty0 447--474, 1967.

\bibitem[Gold(1978)]{gold1978complexity}
E~Mark Gold.
\newblock Complexity of automaton identification from given data.
\newblock \emph{Information and control}, 37\penalty0 (3):\penalty0 302--320, 1978.

\bibitem[Goldman and Mathias(1993)]{goldman1993teaching}
Sally~A Goldman and H~David Mathias.
\newblock Teaching a smart learner.
\newblock In \emph{Proceedings of the sixth annual conference on computational learning theory}, pages 67--76, 1993.

\bibitem[Hilbert and Ackermann(1928)]{hilbert1928grundzuge}
David Hilbert and Wilhelm Ackermann.
\newblock \emph{Grundz{\"u}ge der Theoretischen Logik}.
\newblock J. Springer, Berlin, 1928.

\bibitem[Hopcroft(1971)]{hopcroft1971n}
John Hopcroft.
\newblock An n log n algorithm for minimizing states in a finite automaton.
\newblock In \emph{Theory of machines and computations}, pages 189--196. Elsevier, 1971.

\bibitem[Kleinberg and Mullainathan(2024)]{kleinberg2024language}
Jon Kleinberg and Sendhil Mullainathan.
\newblock Language generation in the limit.
\newblock In \emph{The Thirty-eighth Annual Conference on Neural Information Processing Systems}, 2024.

\bibitem[Malach(2024)]{malach2024auto-regressive}
Eran Malach.
\newblock Auto-regressive next-token predictors are universal learners.
\newblock In \emph{Proceedings of the 41st International Conference on Machine Learning}, ICML'24. JMLR.org, 2024.

\bibitem[Oncina and Garcia(1992)]{oncina1992inferring}
Jos{\'e} Oncina and Pedro Garcia.
\newblock Inferring regular languages in polynomial updated time.
\newblock In \emph{Pattern recognition and image analysis: selected papers from the IVth Spanish Symposium}, pages 49--61. World Scientific, 1992.

\bibitem[Papadimitriou(1993)]{papadimitriou1993computational}
Christos~H Papadimitriou.
\newblock \emph{Computational complexity}.
\newblock Pearson, 1993.

\bibitem[Parekh and Honavar(2001)]{parekh2001learning}
Rajesh Parekh and Vasant Honavar.
\newblock Learning dfa from simple examples.
\newblock \emph{Machine Learning}, 44:\penalty0 9--35, 2001.

\bibitem[Piccinini and Maley(2021)]{sep-computation-physicalsystems}
Gualtiero Piccinini and Corey Maley.
\newblock {Computation in Physical Systems}.
\newblock In Edward~N. Zalta, editor, \emph{The {Stanford} Encyclopedia of Philosophy}. Metaphysics Research Lab, Stanford University, {S}ummer 2021 edition, 2021.

\bibitem[Pitt(1989)]{pitt1989inductive}
Leonard Pitt.
\newblock Inductive inference, dfas, and computational complexity.
\newblock In \emph{International Workshop on Analogical and Inductive Inference}, pages 18--44. Springer, 1989.

\bibitem[Russell and Norvig(2016)]{russell2016artificial}
Stuart~J Russell and Peter Norvig.
\newblock \emph{Artificial intelligence: a modern approach}.
\newblock Pearson, 2016.

\bibitem[Sipser(1996)]{sipser1996introduction}
Michael Sipser.
\newblock \emph{Introduction to the Theory of Computation}, volume~27.
\newblock ACM New York, NY, USA, 1996.

\bibitem[Steffen et~al.(2011)Steffen, Howar, and Merten]{steffen2011introduction}
Bernhard Steffen, Falk Howar, and Maik Merten.
\newblock \emph{Introduction to Active Automata Learning from a Practical Perspective, bookTitle=Formal Methods for Eternal Networked Software Systems: 11th International School on Formal Methods for the Design of Computer, Communication and Software Systems, SFM 2011, Bertinoro, Italy, June 13-18, 2011. Advanced Lectures}, pages 256--296.
\newblock Springer Berlin Heidelberg, address=Berlin, Heidelberg, 2011.
\newblock ISBN 9783642214554.

\bibitem[Sutton and Barto(2018)]{sutton2018reinforcement}
Richard~S. Sutton and Andrew~G. Barto.
\newblock \emph{Reinforcement Learning: An Introduction}.
\newblock A Bradford Book, Cambridge, MA, USA, 2018.

\bibitem[Turing(1936)]{turing1936computable}
Alan Turing.
\newblock On computable numbers, with an application to the entscheidungsproblem.
\newblock \emph{J. of Math}, 58\penalty0 (345-363):\penalty0 5, 1936.

\bibitem[Vaandrager et~al.(2022)Vaandrager, Garhewal, Rot, and Wi{\ss}mann]{vaandrager2022new}
Frits Vaandrager, Bharat Garhewal, Jurriaan Rot, and Thorsten Wi{\ss}mann.
\newblock A new approach for active automata learning based on apartness.
\newblock In \emph{International Conference on Tools and Algorithms for the Construction and Analysis of Systems}, pages 223--243. Springer, 2022.

\bibitem[Wei et~al.(2022)Wei, Wang, Schuurmans, Bosma, Xia, Chi, Le, Zhou, et~al.]{wei2022chain}
Jason Wei, Xuezhi Wang, Dale Schuurmans, Maarten Bosma, Fei Xia, Ed~Chi, Quoc~V Le, Denny Zhou, et~al.
\newblock Chain-of-thought prompting elicits reasoning in large language models.
\newblock \emph{Advances in neural information processing systems}, 35:\penalty0 24824--24837, 2022.

\bibitem[Wittgenstein(1922)]{wittgenstein1922tractatus}
Ludwig Wittgenstein.
\newblock Tractatus logico-philosophicus.
\newblock \emph{Filosoficky Casopis}, 52:\penalty0 336--341, 1922.

\end{thebibliography}

\newpage 
\appendix

\section*{Organization of the Appendix}

Appendix~\ref{app:preliminaries} details the formal framework. Appendix~\ref{app:blackbox} contains the proof for the input-output-learning setting (Section~\ref{sec:blackbox}). Appendix~\ref{app:clock} provides the proof for learning from time-bound observations (Section~\ref{sec:clock}). Appendix~\ref{app:cot} includes the proofs related to learning from policy-trajectory observations (Section~\ref{sec:CoT}).

\section{Supplement to the Formal Framework} \label{app:preliminaries}

\subsection{Finite-State Transducers} \label{app:sec:fst}

Two properties make FSTs elegant objects of study: the existence of a canonical minimal FST given by the Generalized Myhill-Nerode Theorem and the existence of a fast state-distinguishing procedure given by Hopcroft's algorithm (see \citep{steffen2011introduction} and \citep{hopcroft1971n}, respectively). We will briefly consider both.

\begin{theorem}[Generalized Myhill-Nerode] \label{thm:myhill-nerode}
   A function $\chi : A^+ \to B$ is rational if and only if $\chi$ is the semantics of some FST. Moreover, whenever $\chi$ is rational, there exists a unique min-state FST $M_\chi$ with semantics given by $\chi$ and having the following description: $M = (Q, A, B, \d, \g, q_0)$, where $Q = \{[u]_\chi : u \in A^\star\}$, $\d([u]_\chi, a) = [u \cdot a]_\chi, \ \forall u \in A^\star, a \in A$, $\g([u]_\chi, a) = \chi(u \cdot a)$, and $q_0 = [\l]_\chi$.
\end{theorem}
Hence, every rational $\chi$ defines a unique canonical min-state transducer $M_\chi$ with semantics $\chi$ and precisely $|\chi|$ states. And any other FST with these two properties is necessarily isomorphic to $M_\chi$.

Next, we consider the question of distinguishing FST states. Recall that we say that two states $p, q \in Q$ are apart and write $p \# q$ whenever there exists a distinguishing string $\s \in A^+$ s.t. $\g(p, \s) \neq \g(q, \s)$. Interestingly, we can quickly test for state apartness with Hopcroft's algorithm.

\begin{theorem}[Hopcroft's Algorithm] \label{thm:hopcroft}
   Given an FST $M = (Q, A, B, \d, \g, q_0) \in \Phi_A^B$ with $n$ states, there exists an $O(n \log n)$-time algorithm which converts $M$ into the min-state FST $M_{\g_M}$. Moreover, Hopcroft's algorithm finds distinguishing strings of length at most $n-1$ for each $p,q \in Q$ s.t. $p\#q$. Finally, given another FST $M'$ with $n'$ states, Hopcroft's algorithm determines in $O((n+n') \log(n+n'))$ time if $\g_M = \g_{M'}$ and produce a distinguishing string of length at most $n+n'-1$ if not. Finally, Hopcroft's algorithm can test state apartness in $O(n \log n)$ time even for partial FST $M \in \wt{\Phi}_A^B$.
\end{theorem}

\subsection{Ordering of Turing Machines}

We record for future reference in the subsequent proofs the following simple proposition which establishes a computable enumeration of the family of Turing machines $\T$.

\begin{proposition}[Enumeration of TMs] \label{prop:enumeration}
    There exists a computable bijection $\t : \N \to \T$ such that for all $n \in \N$, the number of states in $\t(n)$ increases monotonically with $n$. Moreover, $|\t(n)| = o(\log n)$, and there exists a polynomial-time algorithm that, given a natural number $n$ in binary, outputs a binary representation of $\t(n)$.
\end{proposition}

\subsection{Encodings and Simulations} \label{app:sec:ctt}
We now dedicate a bit of space to highlight an interesting paradox that arises if one takes the Church-Turing Thesis (CTT) seriously. As we discussed in the introduction, the CTT argues that the formal predicate \textit{Turing-computable} completely captures the intuitive notion of computability. Clearly, since not all computable functions operate over domains of strings (e.g., running a cellular automaton for $n$ steps), TMs cannot \textit{implement} every computable function. Instead, the CTT claims that TMs can \textit{simulate} every computable function. In formal terms, according to the CTT, for any \textit{computable} function $f : \D \to \D$ on an enumerable domain $\D$ and any \textit{reasonable} injective encoding $\varphi : \D \inj \SS$, there exists a TM $T \in \T$ such that $\varphi^{-1} \circ T \circ \varphi = f$.\footnote{Notice that if $\varphi$ is intuitively computable, then so is $\varphi^{-1}$ since both $\D$ and $\SS$ are enumerable.} Hence, $f$ is computable if and only if the following diagram commutes for some $T \in \T$.
\begin{equation*}
    \begin{tikzcd}
        \D \arrow[r,"M"] \arrow[hookrightarrow, d,"\varphi"']
        & \D \arrow[hookrightarrow, d,"\varphi"] \\
        \SS \arrow[r,"T"] & \SS
    \end{tikzcd}
\end{equation*}

Now, we need to specify what \textit{reasonable} means. Clearly, we cannot leave $\varphi$ unrestricted. Indeed, suppose $\D = \N_0$ and let $f(n) = \mathbf{1}_\mathcal{P}(n)$, where $\mathbf{1}_\mathcal{P}$ denotes the indicator function for the set of prime numbers $\mathcal{P}$. Also, let $H = \{\langle T, w\rangle_\Sigma \vert T\in\T,  w \in \SS, \text{ and } $T(w)$ $\text{ halts}$\}$, where $\langle T, w\rangle_\Sigma$ denotes some encoding of the pair $(T,w)$ into $\SS$. Now, if $\varphi$ bijectively maps $\mathcal{P}$ into $H$ and $\N_0 \setminus \mathcal{P}$ into $\SS \setminus H$, then since $f$ is intuitively computable, there must exist a TM $T$ deciding the halting problem -- a contradiction. Hence, an encoding should not define an intuitively uncomputable procedure. Therefore, the predicate \textit{reasonable} must also imply \textit{computable} -- a circumstance which pushes us into an infinite regress since our definition of computability relies on reasonable encodings.

We believe that this infinite regress should not cause us to abandon the CTT. Indeed, TMs \textit{do} seem capable of simulating any intuitively computable process under reasonable encodings. However, we also think that reasonable encodings must come from computable procedures. How can we break out of the infinite regress and define reasonability without any reference to computability? The way we propose goes through the old saying ``The limits of my language mean the limits of my world.'' \citep[Proposition 5.6]{wittgenstein1922tractatus}. In a sense, we cannot meaningfully talk about the computable function $f$ or the symbolic domain $\D$ without first embedding $f$ and $\D$ into first-order logic (FOL). Hence, if we denote with $\Xi$ the alphabet of FOL, we can think of some encoding $E : \D \to \Xi$, mapping $w$ to a finite FOL description of $\langle w \rangle_{\mathrm{FOL}}$, as provided to us from the start. Then, we can define an encoding $\varphi : \D \to \SS$ as reasonable if and only if $\varphi$ constitutes a composition of $E$ with some TM taking strings from $\Xi^\star$ to  $\SS$.

\subsection{Tasks as Computable Functions}
We use the words ``function'' and ``task'' interchangeably in order to emphasize the fact that all everyday reasoning or mechanical tasks represent computable functions since a TM can simulate these tasks with an appropriate encoding. For example, brewing a cup of tea specifies the following task $\texttt{BrewTea} : \mathcal{K} \to \mathcal{K}$, where $\mathcal{K}$ denotes the symbolic domain of the kitchen. More precisely, we could think of a kitchen-state $\kappa \in \mathcal{K}$ as a function $\kappa : \texttt{KitchenLattice} \to \texttt{Colors}$ assigning colors to the voxels in the kitchen space. Hence, upon receiving as input a kitchen-state $\kappa$, the computational agent implementing \texttt{BrewTea} goes through a series of kitchen-states and arrives at some $\texttt{BrewTea}(\kappa)$ with voxels that indicate the presence of a freshly brewed cup of tea. Now, let us consider an arbitrary computable encoding $\varphi : \mathcal{K} \to \SS$ which injectively maps kitchen-states to TM-readable $\SS$ representations. Then, invoking the CTT, there exists a TM $T_{\texttt{BrewTea}}$ that simulates the computational path of \texttt{BrewTea} and outputs $\varphi(\texttt{BrewTea}(\kappa))$ when prompted with $\varphi(\kappa)$. In other words, we make the point that a robot with sensors encoding the environment in $\SS$ could learn how to brew a cup of tea in the limit by running the presented algorithms.

\subsection{Characteristic Sets} \label{app:sec:c-sets}

\distinguish*

\begin{proof}
    Assume the contrary: $\forall x \in S_M \cup S_{M'}, \ \eta_M(x) = \eta_{M'}(x)$ and $x \in D_M \cap D_{M'}$. Let $S = S_M \cup S_{M'}$. Then, $E_M(S) = E_{M'}(S)$. Now, by the characteristic property of $S_M$, $\L(E_M(S)) = R$ which computes the same function as $M$. Analogously for $S_{M'}$, $R$ must compute the same function as $M'$ -- contradiction. Hence, there exists an input distinguishing the $\a$-observable behavior of $M$ and $M'$ on the union of their characteristic sets.
\end{proof}
\section{IOO--earning Proofs} \label{app:blackbox}

\blackbox*

\begin{proof}
    We want to show that for all $f \in \mathbf{Q}(\M_\D)$, all input sources $I \subseteq \SS$, and all surjective orderings $w : \N \surj I \cap D_f$ of the example set $E_f(I \cap D_f) = \{(x, f(x)) : x \in I \cap D_f\}$, after some finite time $t^\star = t^\star(f, I, w)$, our learning algorithm $\L$ will converge to a correct representation of $f$. In other words, we want to show that $\forall t \geq t^\star, \ h_{t|I \cap D_f} = f_{|I \cap D_f}$, where $h_t$ is the hypothesis generated by $\L$ at time $t$.

    Let us fix a ground truth $f \in \mathbf{Q}(\M_\D)$, an input source $I \subseteq \SS$, and a surjective ordering of the inputs $w : \N \surj I \cap D_f$ s.t. $I \cap D_f = \{w_t\}_{t=1}^\infty$. We will demonstrate how $\L$ learns $f$ in the limit.

    First, since $f \in \mathbf{Q}(\M_\D)$, there exists $c_1 \in \N$ s.t. $f \in \mathbf{TIME}_\M(Q(c_1, n))$. Hence, there exists a computational model $M \in \M_\D$ which computes $f$ in $c_2 \cdot Q(c,n)$ steps for some constant $c_2 \in \N$. Now, let $\L$ use some reasonable encoding $\varphi : \D \to \SS$ for which the $q$-ECTT kicks in. Using the $q$-ECTT, there exists a TM $\widetilde{T} \in \T$ which simulates $m$ computation steps of $M$ in at most $q(c_3, m)$ steps for some $c_3 \in \N$. Hence, for every $x \in I \cap D_f$, $\widetilde{T}$ starts with $\varphi(x)$ on the tape and halts with $\varphi(f(x))$ after $q(c_3, c_2 \cdot Q(c_1, |x|))$ steps. Let $c = \max(c_1, c_2, c_3)$. Then, $\widetilde{T}$ computes $f$ in $q(c, c \cdot Q(c, n))$ time.

    Second, $\L$ will use the ordering algorithm $\rho$ from Proposition \ref{prop:enumeration} as a subroutine to generate TM representations from $\T$, sorting them by increasing number of states. Let $T_1, T_2, T_3, \dots$ be the ordering of $\T$ due to $\rho$ and let us consider the set
    \begin{equation*}
        H_f = \{T \in \T : \exists C \in \N \text{ s.t. } T \text{ computes } f_{|I \cap D_f} \text{ within } q(C, C \cdot Q(C, n)) \text{ steps}\}.
    \end{equation*}
    Clearly, $H_f$ is nonempty since $\widetilde{T} \in H_f$. Let $K$ be the smallest number assigned by $\rho$ to a TM in $H_f$ and let $c_K \in \N$ be the smallest constant for which $T_K$ computes $f_{|I}$ within $q(c_K, c_K \cdot Q(c_K, n))$ steps. We will prove that $\L$ converges to $T_K$ after finitely many observations.
    
    Now, let us describe the algorithmic strategy $L$ employs on the set of inputs $S_t = \{w_1, \dots, w_t\}$ at time $t \in \N$. The learner $\L$ will maintain a counter $C \in \N$, s.t. initially $C=1$. Given the current value of $C$, for each $w \in S_t$ and for each TM $T \in \{T_1, \dots, T_C\}$, $\L$ will simulate $q(C, C \cdot Q(C, |w|))$ steps of the computation of $T$ on $w$. This simulation is possible, since $\rho, q,$ and $Q$ are all computable functions known to $\L$. We will say that a TM $T_i$ is $\mathbf{(t, C)}$\textbf{-valid} if $i \leq C$ and $T_i$ correctly computes $f(w)$ in time $q(C, C \cdot Q(C, |w|)), \ \forall w \in S_t$. Let $V_{t,C}$ be the set of all $(t, C)$-valid TMs and note that $\L$ computes $V_{t, C}$ through the enacted simulations. Now, if $V_{t,C}$ is nonempty, $\L$ returns the representation of the TM in $V_{t,C}$ with the smallest $\rho$-number. Otherwise, $\L$ continues the same procedure with $C \leftarrow C + 1$. 

    Note that $V_{t+1, C} \subseteq V_{t,C} \subseteq V_{t,C+1}$. Also note that once $C \geq C_K = \max(K, c_K)$, $T_K \in V_{t, C}, \ \forall t \in \N$. Hence, $\L$ follows a finite procedure and outputs a representation consistent with the example set $E_f(S_t)$ at each time $t \in \N$.

    Finally, observe that since $T_K$ is the smallest member of $H_f$ according to the $\rho$-ordering, then for every $i < K$, there exists an input $w_{t(i)}$ such that either $T_i$ takes more than $q(C_K, C_K \cdot Q(C_K, |w|))$ steps to finish computing or $T_i(w_{t(i)}) \neq f(w_{t(i)})$. In other words, $T_i \notin V_{t(i), C_K}$. Hence, let us take $t^\star = \max(t(1), \dots, t(K-1))$. Then, $T_K$ becomes the smallest $\rho$-ordered member of $V_{t^\star, C_K}$ and gets selected by $\L$ for every $t \geq t^\star$.

    \textbf{Three remarks.} First, $\L$ learns a min-state TM which computes $f_{|I \cap D_f}$ in the limit since $\rho$ orders the elements of $\T$ by increasing number of states. Through a different computable ordering, we could have selected for some other characteristic since $\L$ will always learn the first consistent representation with $f_{|I}$.
    
    Second, if we wanted, we could have made $\L(E_f(S_t))$ run in polynomial time in the size of $E_f(S_t)$ through the following cheap trick: Let $\L$ follow the above algorithm, slowly discarding TM representations by checking consistency with each $w \in S_t$. However, when $\L$ reaches some pre-specified $\texttt{poly}(|E_f(S_t)|)$ number of steps, then we modify $\L$ to return the first non-discarded representation from the $\rho$-ordering. Since, as the observation size grows, more computational resources become available, at some finite point in time $\L$ will have discarded all $T_i$ with $i < K$ through the samples $w_{t(i)}$. Hence, $\L$ runs in polynomial time in the observation size and learns $T_K$ in the limit. However, $\L$ no longer provides a representation consistent with the observations at every time step.
    
    Third, we could have also made $\L$ incremental in the sense that $\L$ keeps track of the smallest index $C_t$ of a model not discarded at time $t \in \N$. Since $V_{t+1, C_t-1} \subseteq V_{t,C_t-1} = \es$, $\L$ should start with $C = C_t$ at time $t+1$. The learner $\L$ could also keep track of which representations $T_i$ got discarded due to wrong and not slow computation.
\end{proof}
\section{TBO--Learning Proofs} \label{app:clock}

\clock*

\begin{proof}
We want to show that for all $f \in G_\D$, all models $M \in \M_\D$ computing $f$, all input sources $I \subseteq \D$, and all surjective orderings $w : \N \surj I \cap D_f$ of the example set $E_M(I) = \{(x, f(x), \tau(M, x)) : x \in I \cap D_f\}$, after some finite time $t^\star = t^\star(M, I, w)$, our learning algorithm $\L$ will converge to a correct representation of $f$. In other words, we want to show that $\forall t \geq t^\star, \ h_{t|I \cap D_f} = f_{|I \cap D_f}$, where $h_t$ is the hypothesis generated by $\L$ at time $t$.

Let us fix a ground truth couple $(f, M) \in G_\D \times \M_\D$, an input source $I \subseteq \D$, and a surjective ordering of the inputs $w : \N \surj I \cap D_f$ s.t. $I \cap D_f = \{w_t\}_{t=1}^\infty$. We will demonstrate how $\L$ learns $f$ in the limit.

Now, if $M$ takes $t_M(x)$ steps to compute $f(x)$, then $t_M(x) \leq p \atb(M, x)$ for some constant $p \in \N$. Let $\L$ use some reasonable encoding $\varphi : \D \to \SS$ for which the $q$-ECTT kicks in. Invoking the $q$-ECTT, there exists a TM $\widetilde{T} \in \T$ which simulates $m$ computation steps of $M$ in at most $q(c', m)$ steps for some $c' \in \N$. In other words, on input $\varphi(x)$, $\widetilde{T}$ takes at most $q(c, \ceil{p \atb(M,x)})$ steps to compute $\varphi(f(x))$ for every $x \in \D$. Let $c = \max(c', \ceil{p})$. Then, $\widetilde{T}$ simulates $f(x)$ through $\varphi$ in $q(c, c \atb(M,x))$ time.

The proof now proceeds identically to the proof of \Cref{thm:blackbox}. Hence, we find a polynomial-time (w.r.t the observation size) incremental learning algorithm $\L$ which learns in the limit a minimal-state TM simulating $f_{|I \cap D_f}$ through the encoding $\varphi$. Different encodings will lead to different solutions in the limit. 
\end{proof}
\section{PTO--Learning Proofs} \label{app:cot}

\reduction*

\begin{proof}
    We already covered the reduction part of the lemma. For the learnability part, we give a learning-by-enumeration algorithm $\L$. Let $\chi \in \mathcal{P}_A^B$ be the ground truth, let $I \subseteq A^\star$ be the input source, let $w : \N \surj A^\star$ be a surjective ordering of the inputs, and let $S_t = \{w_1, \dots, w_t\} \subseteq I$ be the sample set at time $t \in \N$. Since the set $\Phi_A^B$ of FSTs computing $\mathcal{P}_A^B$ is recursive, let $\{M_n\}_{n=1}^\infty$ be some computable ordering of the automata in $\Phi_A^B$ used by $\L$. Let $K$ be the index of the first FST in that ordering which generates a rational function agreeing with $\chi$ on $I$. That is, if $\chi_i$ is the rational function due to $M_i$ for every $i \geq 1$, then $\chi_K$ is the one with the smallest index s.t. $\chi_{K|I} = \chi_{|I}$. Clearly, such an index exists since some automaton $M_s$ computes $\chi$. 

    Now, at observation step $t \in \N$, $\L$ goes through the automata $\{M_n\}_{n=1}^\infty$ in order and checks for consistency with $\chi$ on $S_t$. We set $\L$ to return the first encountered FST which agrees with $\chi$ on $S_t$. Hence, at each step $\L$ returns a consistent representation. Also, after finitely many observations, $\L$ will converge to $M_K$. Indeed, since $M_K$ is the first consistent FST on $I$, for each $i < K$, there exists an input $w_{t(i)}$ s.t. $M_i(w_{t(i)}) \neq \chi(w_{t(i)})$. Hence, after time $t^\star = \max(t(1), \dots, t(K-1))$, $M_K$ will be the first consistent representation in the ordering with the sample set.

    The same remarks given at the end of the proof of \Cref{thm:blackbox} apply here. First, the learner $\L$ always learns-in-the-limit the first representation from the ordering consistent with $\chi_{|I}$. Hence, if the ordering lists the automata by increasing size, $\L$ will output the smallest solution. Second, we could make $\L$ run in polynomial-time with respect to the size of the observations $E_\chi(S_t) = \{(w, \chi(w)) : w \in S_t\}$ while not always producing a consistent representation. Third, we could also make $\L$ incremental in the sense of remembering discarded representations from the previous time step.
\end{proof}

\paragraph{Functions $\leftarrow$ Behaviors $\leftarrow$ Representations.} 
For each function $f \in G_\SS$ there exists an infinite set of TMs $\M_f \subset \T$ computing $f$.\footnote{For instance, if $T$ computes $f$, we can let $T_n$ implement $T$ and then cycle through the output $n$ times before halting.} Thus, for any sample $S \subseteq I$, one can observe a multitude of tape behaviors $B_f(S) = \{M[S] : M \in \M_f\}$. Moreover, each $M[S] \in B_f(S)$ generates an infinite equivalence class $[M]_S$ of TMs with the same tape behavior.\footnote{For $T \in \T$, adding dead transitions leads to multiple equivalent representations. In general, $|[T]_\SS| = \infty$.} In our $(\T, \apt)$--LP framework, $\L$ observes $M[S]$ but need not output a TM from $[M]_S$. Indeed, behavior observations should only aid the learning of $f$ and not become the object of learning. Ideally, the size of the smallest TM in $[M]_I$—the set of consistent representations in the limit—should determine the learning complexity.

\unbounded*

\begin{proof}
    If $L_\Sigma \subset \T$ denotes the class of linear-time running halting TMs, we will prove the stronger statement that there cannot exist characteristic sets of $\b$-bounded mass for $L_\Sigma$ under $\apt$-observations. We will argue by contradiction. Suppose that for the algorithm $\L$ learning $\F(L_\Sigma)$ in the limit from $\apt$-observation of $L_\Sigma$, there exist $\b$-bounded characteristic sets such that $\n{S_T} \leq \b(T), \forall T \in L_\Sigma$. Then, we will show the existence of a TM solving the Halting Problem.

    For a TM $T \in \T$ and $w \in \SS$, we construct the pair of TMs $Y_{T,w}$ and $N_{T,w}$ specified as follows. On input $x \in \SS$, $Y_{T,w}$ and $N_{T,w}$ both simulate the execution of $T(w)$ for $|x|$ steps, and if $T$ halts, $Y_{T,w}$ outputs 1 and halts and $N_{T,w}$ outputs $0$ and halts. If $T$ does not halt on $w$ for $|x|$ steps, then both $Y_{T,w}$ and $N_{T,w}$ halt immediately, having the exact same tape behavior $Y_{T,w}[x] = N_{T,w}[x]$. Clearly, $\{Y_{T,w}, N_{T,w} : T \in \T, w \in \SS\} \subset L_\Sigma$. Hence, $\forall T \in \T \forall w \in \SS$, $\n{S_{Y_{T,w}}} \leq \b(Y_{T,w})$ and $\n{S_{N_{T,w}}} \leq \b(N_{T,w})$. 
    
    For every $T \in \T$ and for every $w \in \SS$, let $\langle T, w \rangle$ denote the encoding of the pair $(T, w)$ in $\SS$. Now, we define the following TM $H \in \T$ solving the Halting Problem. Upon receiving $\langle T, w \rangle$ as input, $H$ constructs representations of the TMs $Y_{T,w}$ and $N_{T,w}$ and runs them on $\b$, thereby producing the bound $B = \max(\b(Y_{T,w}), \b(N_{T,w}))$. We know from \Cref{lem:distinguish} that if $Y_{T,w}$ and $N_{T,w}$ compute different functions, then there must exist an $x \in S_{Y_{T,w}} \cup S_{N_{T,w}}$ such that $\eta_{Y_{T,w}}(x) \neq \eta_{N_{T,w}}(x)$. In other words, if $T$ halts, then $Y_{T,w}(x) \neq N_{T,x}(x)$ for some $x \in S_{Y_{T,w}} \cup S_{N_{T,w}}$. If $T$ does not halt, however, then
    \begin{equation*}
        \eta_{Y_{T,w}}(x) = (x, Y_{T,w}(x), Y_{T,w}[x]) = (x, N_{T,w}(x), N_{T,w}[x]) = \eta_{N_{T,w}}(x), \ \forall x \in \SS.
    \end{equation*}
    Hence, $H$ just needs to check if $Y_{T,w}$ and $N_{T,w}$ agree on $S_{Y_{T,w}} \cup S_{N_{T,w}}$ to determine whether $T$ halts on $w$.

    Now, since $\n{S_{Y_{T,w}}}, \n{S_{N_{T,w}}} \leq B$, then $\forall x \in S_{Y_{T,w}} \cup S_{N_{T,w}}, \ |x| \leq B$. Therefore, $S_{Y_{T,w}} \cup S_{N_{T,w}} \subset \Sigma^B$. Thus, $H$ can simulate $Y_{T,w}$ and $N_{T,w}$ on all inputs from $\Sigma^B$ to solve the Halting Problem for $(T, w)$ -- contradiction.

    Similarly, if there existed a computable function $\mathfrak{C} : \T \to \SS$ assigning characteristic sets to representations in $\T$ for the learner $\L$, then $H$ could have used $\mathfrak{C}$ to construct $S_{Y_{T,w}} \cup S_{N_{T,w}}$ and check for a distinguishing string $x$ which would prove whether $T$ halts on $w$.
\end{proof}

\paragraph{The Correct Measure for Learning Efficiency.}
In light of the proof of \Cref{thm:unbounded}, one might argue that learning efficiency should be measured not by the \textit{mass} of characteristic sets $\n{S_M}$ but by their \textit{size} $|S_M|$. Indeed, the proof shows that a learner cannot distinguish between the functions computed by $Y_{T,w}$ and $N_{T,w}$ until the input length $|x|$ is large enough for $T$ to halt. If $T$ requires exponential time in the representation sizes of $Y_{T,w}$ and $N_{T,w}$ to halt on $w$, it is unreasonable to penalize the learner for needing a long input $x$ to learn. Instead, the burden of hard examples should fall on the teacher. Thus, an efficient learner should require \textit{few}, not necessarily \textit{short}, examples. 
This raises the open question of whether a polynomial-time algorithm exists for learning $G_\SS$ in the limit from policy-trajectory observations of $\T$ with polynomial-size characteristic sets.

Let us define a model class $\M$ as \textbf{$\apt$-restrictively identifiable in polynomial time and samples} ($\apt$-RIPTS) if there exists a polynomial-time learner $\L$ for which every $M \in \M$ admits a characteristic set with polynomial size $|S_M| = \texttt{poly}(|M|)$ regardless of the input source $I$. Hence, due to the reduction established by \Cref{thm:reduction}, a more interesting and general open question becomes if the rational function class $\mathcal{P}_A^B$ is $\eps$-RIPS: i.e., RIPS from input-output observations.

\MSM*

\begin{proof}
    Let $f \in \CC_\SS$. We will prove the existence of an alphabet $\G_f \supset \Sigma$ and a TM $T_f \in \mathrm{T}_\Sigma^{\G_f}$ which computes $f$ and whose tape-behavior observations allow the MSM algorithm to learn $f$ in the limit for any input source $I \subseteq \SS$.
    
    Let $\G = \Sigma \cup \{\l\}$ and recall Definition \ref{def:TM} and the discussion underneath for the architectural specifics of our class of TMs $\T$. In particular, note that the transition function allows for the TM head to stay in-place. As proved in many textbooks on complexity and computability \cite{papadimitriou1993computational, sipser1996introduction, arora2009computational}, a TM with a transition function requiring constant left or right movement can compute any recursive function on $\SS$. Let $T_f'= (Q', \G, \d', q_0, q_f) \in \T$ compute $f$ and have a transition function that never stays in-place. We will now describe the transformation of $T_f'$ into $T_f$. Let $Q' = \{q_1, \dots, q_n\}$. We arbitrarily order the $m$ transitions of $T_f'$ as $\{e_i\}_{i=1}^m$, and if $e_k = q_i \xrightarrow{\s : \s', D} q_j$, we add a dummy states $p_k$ and substitute $e_k$ in the transition diagram with the path $$q_i \xrightarrow{\s : \g_k, S} p_k \xrightarrow{\g_k : \s', D} q_j,$$ where $D \in \{L,R\}$ and where $\g_1, \dots, \g_m \notin \G$. We let $\G_f = \G \cup \{\g_1, \dots, \g_m\}$, and  we assign some other arbitrary transitions to the states $p_k$ for every symbol in $\G_f$ to obtain a well-defined transition function $\d$ on the set of states $Q = Q' \cup \{p_k\}_{1 \leq k \leq m}$. Then, we define $T_f$ as $T_f = (Q,\G_f,\d,q_0,q_f)$. Clearly, $T_f$ also computes $f$, but now each computation takes twice as long.

    Suppose now that we receive behavior observations from $T_f$ from an input source $I$. Let $\psi(T_f) \in \Phi_A^B$ correspond to the FST with the same transition diagram as $T_f$ where $A = \G_f$, $B = \G_f \cup \{L,R,S\}$. Moreover, let $M = I(\psi(T_f)) \in \wt{\Phi}_A^B$ denote the partial automaton produced from $\psi(T_f)$ by removing the unused transitions when all inputs come from $I$. We will show that MSM learns the representation of $M$ in the limit, which proves that MSM learns a hypothesis function $h$ agreeing with $f$ on $I$.

    Let $t^\star$ denote the time at which the sample $S_{t^\star} \subseteq I$ exercises all transition of $M$. We prove that for $t \geq t^\star$, MSM$(S_t) \equiv_I M$. Let $Q_M$ be the set of states of $M$ and let $Q_T$ be the set of states of the observation tree $\mathcal{T}_M(S_t)$. Also, let $\phi : Q_T \to Q_M$ assign the `true' value to the tentative states of $\mathcal{T}_M(S_t)$. In other words, if $\texttt{path}(q)$ corresponds to the input path leading from the root of $\mathcal{T}_M(S_t)$ to $q \in Q_T$, then $\phi(q) = \d_M(q_0, \texttt{path}(q))$.

    Now, with a slight abuse of notation, let $X = Q_M \cap Q'$ -- i.e., the computing states of $T_f'$ -- and let $Y = Q_M \cap \{p_i^j\}_{1\leq i \leq j \leq m}$ -- i.e., the dummy states. Note that if $\phi(p), \phi(q) \in X$ and $\phi(p) \neq \phi(q)$, then either there exists a distinguishing string for $p$ and $q$ in $\mathcal{T}_M(S_t)$, or $s(p, q) = 1$. Moreover, if $\phi(p), \phi(q) \in Y$ and $\phi(p) \neq \phi(q)$, then $\mathcal{T}_M(S_t)$, or $s(p, q) = 1$. Hence, first all of the states in $\mathcal{T}_M(S_t)$ corresponding to the same states in $X$ will get merged due to higher similarity. Only after all of the states in $\phi^{-1}(X)$ get correctly merged, will the MSM algorithm consider other mergers which will only lead to equivalent to $M$ automata under the input source $I$.

    \textbf{Remark.} Note that MSM needs at most $m$ samples to learn $T_f$ -- one sample to cover each transition.
\end{proof}


\end{document}